\documentclass[conference]{IEEEtran}
\IEEEoverridecommandlockouts
\usepackage{cite}
\usepackage{amsmath,amssymb,amsfonts}
\usepackage{algorithmic}
\usepackage{graphicx}
\usepackage{textcomp}
\usepackage{xcolor}
\usepackage{amsthm}
\usepackage{float}
\newcommand{\bigO}{\mathcal{O}}
\usepackage{caption} 
\usepackage{subcaption} 
\newtheorem{thm}{Theorem}[section]

\newtheorem{lem}[thm]{Lemma}

\newtheorem{cor}[thm]{Corollary}

\def\BibTeX{{\rm B\kern-.05em{\sc i\kern-.025em b}\kern-.08em
    T\kern-.1667em\lower.7ex\hbox{E}\kern-.125emX}}
\begin{document}

\title{Analysis of Lower Bounds for Simple Policy Iteration\\
\thanks{1. Department of Electrical Engineering, IIT Bombay, India. \{sconsul, bhishma, kumar.ashutosh, parthasarathi.k\}@iitb.ac.in}
\thanks{\# Equal Contribution.}
\thanks{Guided by Prof. Shivaram Kalyanakrishnan}
\thanks{This work was undertaken as a part of a course project in CS 747,
  Autumn 2019, at IIT Bombay. Another team (N. Joshi, S. Shah, N. Jain, S. Balasubramanian)
  independently came up with a similar lower bound and analysis}
}

\author{
\IEEEauthorblockN{Sarthak Consul$^{1, \#}$}
\and 
\IEEEauthorblockN{Bhishma Dedhia$^{1, \#}$}
\and 
\IEEEauthorblockN{Kumar Ashutosh$^{1, \#}$}
\and 
\IEEEauthorblockN{Parthasarathi Khirwadkar$^{1, \#}$}
}
\maketitle

\begin{abstract}
We prove a strong exponential lower bound of $\bigO\big((3+k)2^{N/2-3}\big)$ on simple policy iteration for a family of $N$-state $k$-action MDP.
\end{abstract}


\section{\textbf{Introduction}}

Policy iteration is a family of algorithms that are used to find an optimal policy for a given Markov Decision Problem (MDP). Simple Policy iteration (SPI) is a type of policy iteration where the strategy is to change the policy at exactly one improvable state at every step. Melekopoglou  and Condon [1990] \cite{Melekopoglou1994OnTC} showed an exponential lower bound on the number of iterations taken by SPI for a 2 action MDP. The results have not been generalized to $k-$action MDP since. 

In this paper we revisit the algorithm and the analysis done in \cite{Melekopoglou1994OnTC}. We generalize the previous result and prove a novel exponential lower bound on the number of iterations taken by policy iteration for $N-$state, $k-$action MDPs. We construct a family of MDPs and give an index-based switching rule that yields a strong lower bound of $\bigO\big((3+k)2^{N/2-3}\big)$. In section 2 we describe the relevant background and in section 3 we present the important notations. In section 4 we show the MDP construction and in section 5 we describe the index based switching rule. This is followed by various experiments in section 6 and a proof in section 7. 

\section{\textbf{Background}}
A Markov Decision Process (MDP) \cite{Bellman:1957},\cite{Puterman:1994} represents the environment for a sequential decision making problem. 
An MDP is defined by the tuple $(S,A,T,R, \gamma)$. $S$ is the set of states and $A$ is the set of actions. $|S|$ is denoted as n and $|A|$ is denoted as k. $T:S\times A \times S \rightarrow [0,1]$ is a function which gives the transition probability. Specifically $T(s,a,s')= $ Probability of reaching state $s'$ from state $s$ by taking action $a$. $R:S\times A \rightarrow \mathbb{R}$ is the reward function. $R(s,a)$ is the expected reward that the agent will get by taking action $a$ from state $s$. $\gamma$ is the discount factor, which indicates the importance given to future expected reward. 
Policy $\pi:S\times A \rightarrow \mathbb{R}$ is defined as the probability that agent will choose action $a$ from state $s$. If the policy is deterministic, then $\pi(s)$ is the action that the agent takes when it is on state $s$. For a given policy $\pi(s)$, we define the value function $V^\pi(s):S \rightarrow \mathbb{R}$ as the total expected reward that the agent receives by following the policy from state $s$. The state-action value function $Q^\pi(s,a):S\times A \rightarrow \mathbb{R} $ for a policy $\pi(s)$ is defined as the total expected reward that the agent receives if it takes action $a$ from state $s$, and then follows the policy $\pi$.

\subsection{\textbf{Policy Iteration}}
Policy Iteration (PI) is an iterative algorithm that is used to obtain the optimal policy of a MDP, Let $(S,A,P,R,\gamma)$ describe a MDP and let $\Pi$ be the set of all policies. It has been proved that there exists an optimal policy $\pi^{*}$ such that $\forall \pi \in \Pi, s \in S, V^{\pi^{*}}(s)\geq V^{\pi}(s) $. PI consists of two fundamental steps performed sequentially in every iteration:\\
\textit{Policy Evaluation:} This step is used to evaluate the state values of the MDP under a particular policy. Given a deterministic policy $\pi$ that maps $S \to A$, the state values satisfy the following relation $\forall s \in S$:
\begin{equation*}
    V^{\pi}(s) = \Sigma_{s' \in S} T(s,\pi(s),s')(R(s,\pi(s),s')+\gamma V^{\pi}(s'))
\end{equation*}
The state values can be computed by solving the system of linear equations.
\\
\textit{Policy Improvement:} The state-action value function can be found using:
\begin{equation*}
    Q^{\pi}(s,a) =  \Sigma_{s' \in S} T(s,a,s')(R(s,a,s')+\gamma V^{\pi}(s'))
\end{equation*}
A state is defined as improvable state under a policy $\pi$ if $\exists a \in A \text{ such that }Q^{\pi}(s,a) > Q^{\pi}(s,\pi(s))$. One or more improvable states are switched to an improvable action under the policy improvement step and the resultant policy can be denoted as $\pi^{'}$. There can be many choices for the "locally improving" policies $\pi^{'}$ and different PI variants follow different switching strategies. 

\section{\textbf{Notation}}

Starting from time $t=0$, policy iteration would update the policy to $\pi^t(s)$. The value function and Q-value corresponding to this $\pi^t(s)$ is denoted as $V^t(s)$ and $Q^t(s,a)$ respectively.

The policy $\pi^t(s)$, for the n-states for any general t is expressed as $S_n S_{n-1}\ldots S_1$.

\section{\textbf{MDP Construction}}
In this section we formulate a method of constructing a family $\mathcal{F}$ of MDPs, that give the lower bounds for the switching procedure. Our MDP construction builds over the formulation given by Melekopoglou  and Condon [1990]. For a MDP having $N$ states, $k$ ($\geq 2$) actions the MDP graph is as follows:
\begin{itemize}
    \item The graph has $n$ ($1, \ldots ,n$) ``state" vertices
    \item  The graph has $n$ ($1', \ldots ,n'$) ``average" vertices
    \item The graph has 2 sink (terminal) vertices with sink value $(\alpha,\beta)$ =$ (-1,0)$. 
\end{itemize}
where $n=N/2-1$.\\
The transitions for actions ($0 \ldots k-1$) on the MDP graph are constructed as follows:
\begin{itemize}
    \item Every action taken on an average vertex $s'$ results into an equally likely transition to the state vertex $s-2$ and average vertex $(s-1)'$ 
    \item Action 0 on a state vertex $s$ results into a deterministic transition to the state vertex $s-1$
    \item Action 1 on a state vertex $s$ results into a deterministic transition to the average vertex $s'$
    \item In a MDP having $k>$2: An action  $\mathcal{A}\in \{2,3, \ldots, k-1\}$ on the state vertex $n$ results into a deterministic transition to the average vertex $n'$.
    \item  In a MDP having $k>$3: The action  $k-1$  on a state vertex $x$ $\in (1, \cdots ,n-1)$ results into a deterministic transition to the average vertex $(x+1)'$.
    \item In a MDP having $k>$3: An action  $\mathcal{A}\in \{2,3, \ldots, k-2\}$ on a state vertex $x \in (1, \ldots ,n-1)$ results in a stochastic transition to the average vertex $(x+1)'$ with a probability
    $p_{\mathcal{A}}\in (0,1)$ and to the average vertex $x'$ with a probability $1-p_{\mathcal{A}} = q_{\mathcal{A}}$. An increasing order is maintained over the transition probabilities, that is $p_{\mathcal{A}} > p_{\mathcal{A}-1}$.\\

\end{itemize}

Every transition into the sink states gives a reward equal to the sink value. Every other transition gives a reward of 0. The MDP is undiscounted and $\gamma$ is set to 1. Note that setting $k$ equal to 2 gives the family of MDPs described by Melekopoglou  and Condon \cite{Melekopoglou1994OnTC}. We shall denote the $n$ states, $k$ actions MDP belonging to this family as $\mathcal{F}(n,k)$ henceforth.

Clearly, PI will never update the policy at the average vertices as due to the equivalence of all the actions for the average vertices and so their policy will always be their initial policy. Thus for all subsequent analysis, the only the policy of the state vertices $S_n S_{n-1}\ldots S_1$ are considered.

Note that the optimal policy for this MDP is ($S_n S_{n-1}\ldots S_1 =00 \ldots 01$). 

Figures~\ref{fig:mdp_2_3}, \ref{fig:mdp_3_3} and \ref{fig:mdp_3_5} show the MDP graph for $(n=2,k=3)$, $(n=3,k=3)$ and $(n=3,k=5)$ respectively. 
\begin{figure*}[htbp]
\centering
    \includegraphics[scale = 0.24]{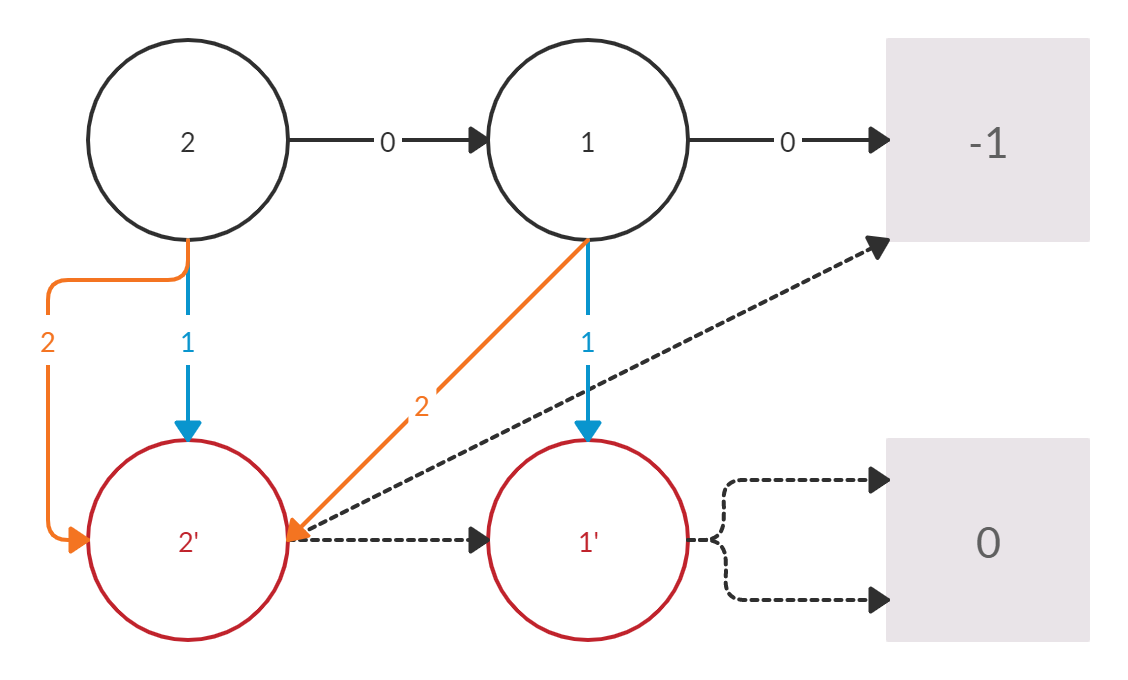}
    \caption{$\mathcal{F}(2,3)$, Sinks are square vertices (Best viewed in color)}
    \label{fig:mdp_2_3}
\end{figure*}

\begin{figure*}[htbp]
    \centering
    \includegraphics[scale = 0.3]{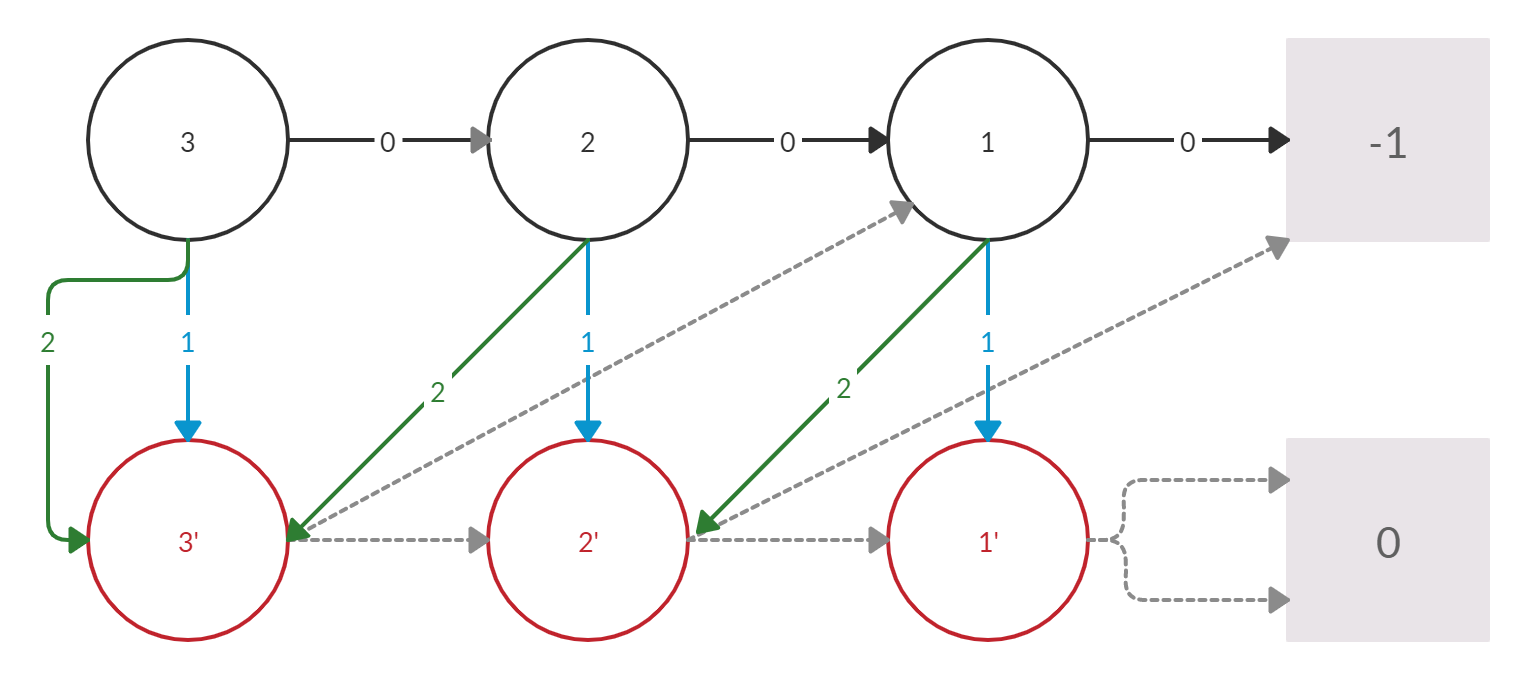}
    \caption{$\mathcal{F}(3,3)$ (Best viewed in color)}
    \label{fig:mdp_3_3}
\end{figure*}

\begin{figure*}[htbp]
    \centering
    \includegraphics[scale = 0.3]{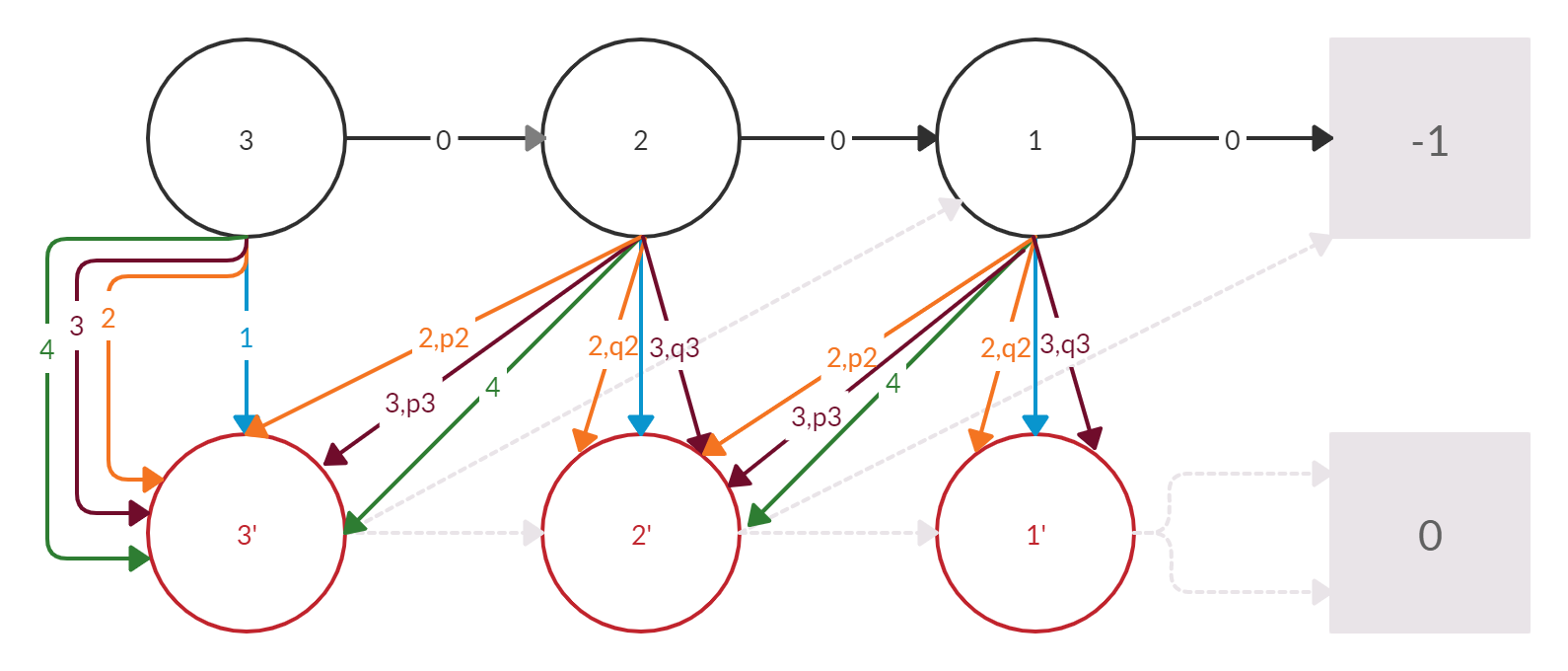}
    \caption{$\mathcal{F}(3,5)$ (Best viewed in color)}
    \label{fig:mdp_3_5}
\end{figure*}

\section{\textbf{Simple Policy Iteration}}
In Simple Policy Iteration (SPI), the policy of an arbitrary improvable state is switched to an arbitrary improving action. Specifically, the improvable state with the highest index is selected and its policy is switched to the improvable action with the highest index.

We denote the number of iterations taken by SPI for the a $n$-state, $k$-action MDP from the above family with an initial policy of $S_n S_{n-1}\ldots S_2 S_1 = 00\ldots00$ to converge to the optimum policy as $\mathcal{N}(n,k)$. We shall experimentally show and later prove that:
\begin{equation}
 \mathcal{N}(n,k) = (3+k)2^{n-2} -2   
\end{equation}
\section{\textbf{Experiments}}

Figure~\ref{fig:plot_nk} shows a plot of the number of iterations against the number of states and actions. Table~\ref{tab:nk} in the appendix contains number of iterations for for all $n,k$ pairs upto $n=k=10$.

\begin{figure}[H]
\includegraphics[width=\linewidth]{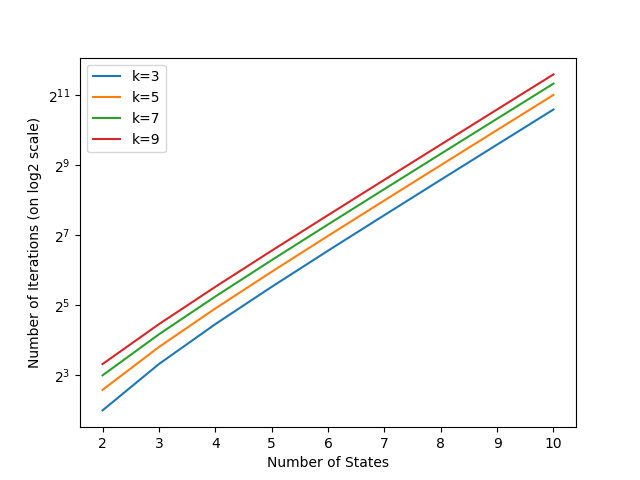}
\includegraphics[width=\linewidth]{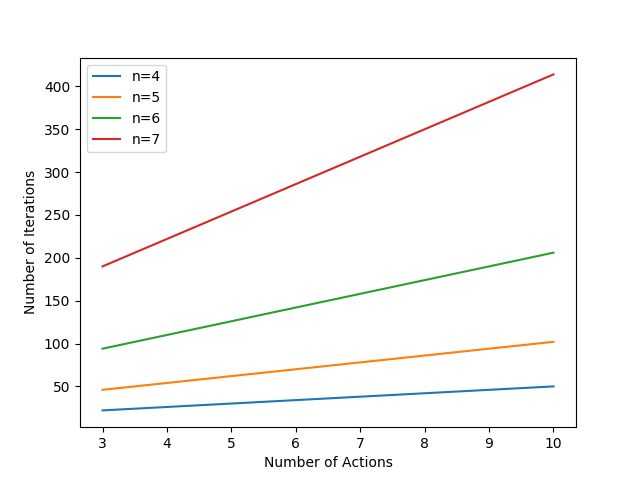}
\caption{(Top) Variation of number of iterations with number of states for fixed number of actions. The vertical scale is in logarithm with base 2. The line is almost linear which justifies the exponential dependence with number of states. (Bottom) Variation of number of iterations with number of action for fixed k. This dependence is linear with an increasing slope.}
\label{fig:plot_nk}
\end{figure}

We next describe how the switching happens for the MDP graph $(n=2,k=3)$ shown in Figure 1. Initially both vertices 1 and 2 are switchable because:\\
\begin{align*}
   Q(2,1) = Q(2,2) = -\frac{1}{2} > Q(2,0) = -1\\
Q(1,1) = 0 > Q(1,2) = -\frac{1}{2} > Q(1,0) = -1 
\end{align*}

According to the switching rule, state 2 switches to action 1. At this point only state 1 is improvable. So state 1 switches to the highest indexed improvable action which is 2. After this switch state 1 is again the only improvable state and it switches to action 1. This switch causes state 1 to attain its best possible value (0) and also makes state 2 improvable because:
$$Q(2,0) > Q(2,2) = Q(2,1) -\frac{1}{2}$$
Hence state 2 switches to action 0 and SPI converges to the optimal policy. The switching has been shown in the table below. 
\begin{table}[h]
\centering
 \caption{Switching sequence for $n=2,k=3$}
\begin{tabular}{ |c|c|c c| } 
 \hline
  t & $\pi(2) \quad \pi(1)$ &  $ V^t(2) $ & $ V^t(1)$ \\ 
 \hline
0& $0\quad 0 $ & $-1 $ & $-1 $\\[0.1cm]
1& $2\quad 0 $ & $-\frac{1}{2} $&$ -1 $\\[0.1cm]
2& $2\quad 2 $ & $-\frac{1}{2} $&$ -\frac{1}{2} $\\[0.1cm]
3& $2\quad 1 $ & $-\frac{1}{2} $&$ 0 $\\[0.1cm]
4& $0\quad 1 $ & $ 0$ &$ 0 $\\[0.1cm]
\hline
\end{tabular}
    \label{tab:switch_seq_n2}
\end{table}


\section{\textbf{Proof}}

The proof of the recursive relation requires the construction of a complementary family of MDPs, $\mathcal{F}_C$, which have the same structure and transition probabilities as $\mathcal{F}$ but sink values of the opposite relative order. We shall denote the $n$ states, $k$ actions MDP belonging to this complementary family as $\mathcal{F}_C(n,k)$ henceforth. By Corollary~\ref{lem:sink_val}, the complementary MDP is set to have sink rewards of (1,0). Note that the optimal policy for $\mathcal{F}_C(n,k)$ is $S_n S_{n-1}\ldots S_2 S_1 = 00\ldots00$. We denote the number of iterations taken by SPI for a $n$-state, $k$-action complementary MDP beginning with a policy of $S_n S_{n-1}\ldots S_2 S_1 = 00\ldots01$ as $\mathcal{N}_C(n,k)$. 

\begin{lem}
\label{lem:sink_invariant}
Policy iteration for the n-state, k-action MDP from family $\mathcal{F}$ and $\mathcal{F}_C$ is invariant to the actual value of sink values, and only depends on their relative values
\end{lem}
\begin{proof}
Let sink values be $(\alpha,\beta)$ = ($k_1$,$k_2$).\\
The transformation of the sink reward maintaining the relative order can be expressed as by a linear transform: $$\mathcal{T}: (k_1,k_2)\rightarrow (A\times k_1+B,A\times k_2+B)$$ where $A\in \mathbb{R}^+, B\in \mathbb{R}$.\\
The linear transformation of the sink rewards would result in the same transformation to the $Q$ and $V$ values. As $A>0$, the relative orders in Q-values do no change and so the switches do not change.\\
\end{proof}

\begin{cor}
\label{lem:sink_val}
Sink values for the MDPs from family $\mathcal{F}$ can be set to $(\alpha,\beta)$ = (-1,0) and those for the complementary MDPs can be set to $(\alpha,\beta)$ =(1,0) without loss of generality.
\end{cor}

\begin{lem}
\label{lem:Q_1}
At any time t, for $\mathcal{F}(n,k)$, 
\begin{equation*}
    Q^{t}(1, 1) > Q^{t}(1, 2) > Q^{t}(1, 3) > \ldots > Q^{t}(1, k-1)> Q^{t}(1, 0)
\end{equation*}
\end{lem}
\begin{proof}
By the structure of the MDP,
$V^t(2')=-\frac{1}{2}$ and $V^t(1')=0$\\
This results in
\begin{equation*}
    Q^t(1,i) = \begin{cases}
    -1, & \text{ if } i=0\\
    0, & \text{ if } i=1\\
    -1/2, & \text{ if } i=k-1\\
    \end{cases}
\end{equation*}
By the construction of $\mathcal{F}$,
\begin{align*}
    p_i &\in (0,1) \quad \forall i \in \{2,3,..k-1\}\\
    p_j &> p_i \quad \forall j>i, i,j \in \{2,3,..k-1\}
\end{align*}
Plugging the values of $Q^t(1,\cdot)$ will yield the desired relation.
\end{proof}

\begin{lem}
\label{lem:QC_1}
At any time t, for $\mathcal{F}_C(n,k)$, 
\begin{equation*}
    Q^{t}(1, 0) > Q^{t}(1, k-1) > Q^{t}(1, k-2) > \ldots > Q^{t}(1, 2)> Q^{t}(1, 1)
\end{equation*}
\end{lem}
\begin{proof}
By the structure of the MDP,
$V^t(2')=\frac{1}{2}$ and $V^t(1')=0$\\
This results in
\begin{equation*}
    Q^t(1,i) = \begin{cases}
    1, & \text{ if } i=0\\
    0, & \text{ if } i=1\\
    p_i/2 ,& \forall i>1
    \end{cases}
\end{equation*}
By the construction of $\mathcal{F}_C$,
\begin{align*}
    p_i &\in (0,1) \quad \forall i \in \{2,3,..k-1\}\\
    p_j &> p_i \quad \forall j>i, i,j \in \{2,3,..k-1\}
\end{align*}
Plugging the values of $Q^t(1,\cdot)$ will yield the desired relation.
\end{proof}

\begin{lem}[Baseline]\label{lem:baseline}
$$\mathcal{N}(2,k) = k+1  \quad \forall k\geq 3$$  
\end{lem}
\begin{proof}
We have the initial policy $\pi^0 = S_{2} S_{1} = 00$ and
\begin{equation*}
    Q^{0}(2,0)=-1 \text{ and } Q^0(2, a)=-\frac{1}{2} \quad \forall a>0
\end{equation*}
As per our definition of SPI, policy with highest improvable action will be chosen. Consequently,
\begin{equation*}
    \pi^1 = S_{2} S_{1} = k-1, 0
\end{equation*}
Next we have $Q^{1}(2, k-1) \geq Q(2, a)  \quad \forall a$, we focus on improving state 1. We observe
\begin{equation*}
    Q^{1}(1, 0)=-1 \text{ and } Q^{1}(1, k-1)=-\frac{1}{2}
\end{equation*}
Hence,
\begin{equation*}
    \pi^2 = S_{2} S_{1} = k-1, k-1
\end{equation*}
Even now, state 2 is not improvable and we have 
\begin{equation*}
    Q^{2}(1, k-2)=\frac{-p_{k-2}}{2} \geq \frac{-1}{2}
\end{equation*}
 by our choice of $p_{k-2}$. Hence, 
 \begin{equation*}
     \pi^3 = S_{2} S_{1} = k-1, k-2 
 \end{equation*}
 Now,
 \begin{equation*}
    Q^{3}(2, 0) = V^3(1) = \frac{-p_{k-2}}{2} > \frac{-1}{2} = Q^{3}(2, a) \quad \forall a \neq 0 
 \end{equation*}
 So,
 \begin{equation*}
    \pi^4 = S_{2} S_{1} = 0(k-2) 
 \end{equation*}
 For all consequent $i^{th}$ iteration,
 \begin{equation*}
    Q^{i}(2, 0) = V^i(1) >  \frac{-1}{2} \text{ but } Q^{i}(2, a) = V^i(2') = \frac{-1}{2}
 \end{equation*}
   and hence only $S_{1}$ is improvable. From Lemma \ref{lem:QC_1}, we have
\begin{equation*}
    Q^{4}(1, 1) > Q^{4}(1, 2) > \ldots Q^{4}(1, k-2)
\end{equation*}
 Thus, the next $k-3$ iterations are required to reach
\begin{equation*}
    \pi^{k+1} = S_{2} S_{1} = 01
\end{equation*}
thus giving a total of $k+1$ iterations.

\end{proof}

\begin{lem}
\label{lem:FC_1}
For the $\mathcal{F}_{C}(n,k)$ with initial policy: $S_n S_{n-1}\ldots S_2 S_1 = 00\ldots01$  it takes exactly $\mathcal{N}(n-1,k)$ iterations before policy at the first state vertex changes.
\end{lem}
\begin{proof}
Due to the switching rule of SPI, $S_1$ will only change when all the other state vertices are not improvable. Until all the higher states finish improving, the current sinks and state 1 can be effectively reduced to new sinks with $(\alpha,\beta) = (0,\frac{1}{2})$. This reduction is shown in Figure~\ref{fig:Fc_red}. Using Lemma~\ref{lem:sink_invariant} the resultant MDP is equivalent to $\mathcal{F}(n-1,k)$ having initial policy $S_n S_{n-1}\ldots S_2 = 00\ldots00$. This MDP takes $N(n-1,k)$ iterations to converge to the optimal policy $S_n S_{n-1}\ldots S_2 = 00\ldots01$. By this logic after the $\mathcal{N}(n-1,k)$ iterations, the policy would be  $S_n S_{n-1}\ldots S_2 S_1 = 00\ldots11$
\end{proof}

\begin{figure}[htbp]
    \centering
    \includegraphics[width=\linewidth]{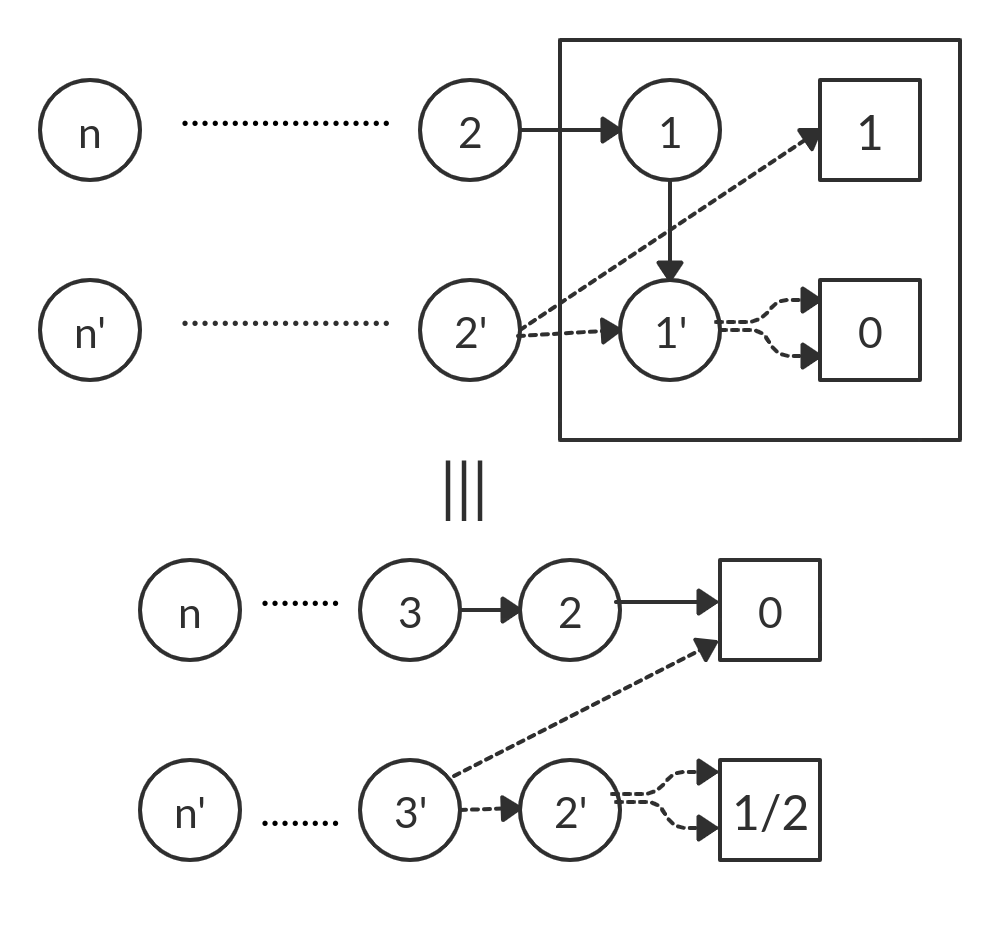}
    \caption{ Reducing $\mathcal{F}_C(n,k)$ to $\mathcal{F}(n-1,k)$ with respect to its initial policy}
    \label{fig:Fc_red}
\end{figure}

\begin{lem}
\label{lem:FC_2}
If $\pi^t = S_n S_{n-1}\ldots S_2 S_1 = 00\ldots11$ for $\mathcal{F}_{C}(n,k)$, the next 2 switches of simple policy iteration occur at state 1
\end{lem}
\begin{proof}
Using Lemma~\ref{lem:FC_1} the policy $\pi^t$ is optimal with respect to the states $n,n-1 \ldots 3,2$ and is $S_n S_{n-1}\ldots S_2 = 00\ldots01$. Vertex 1 is the only improvable state and according to Lemma~\ref{lem:QC_1} it switches to its highest indexed improvable action $k-1$. Hence the policy becomes  $S_n S_{n-1}\ldots S_2 S_1 = 0,0\ldots1,k-1$. With respect to the current policy the sinks and state 1 can be effectively reduced to new sinks with $(\alpha,\beta) = (\frac{1}{2},\frac{1}{2})$. This reduction is shown in Figure~\ref{fig:Fc_red_2}. The  action values are equal for equal sink values and hence the policy for states $n,n-1 \ldots 3,2$ is still optimal. State 1 is still the only improvable state and according to Lemma~\ref{lem:QC_1} it switches to the next improvable action 0. This completes the proof. The policy would now be  $S_n S_{n-1}\ldots S_2 S_1 = 00\ldots10$
\end{proof}

\begin{figure}[htbp]
    \centering
    \includegraphics[width=\linewidth]{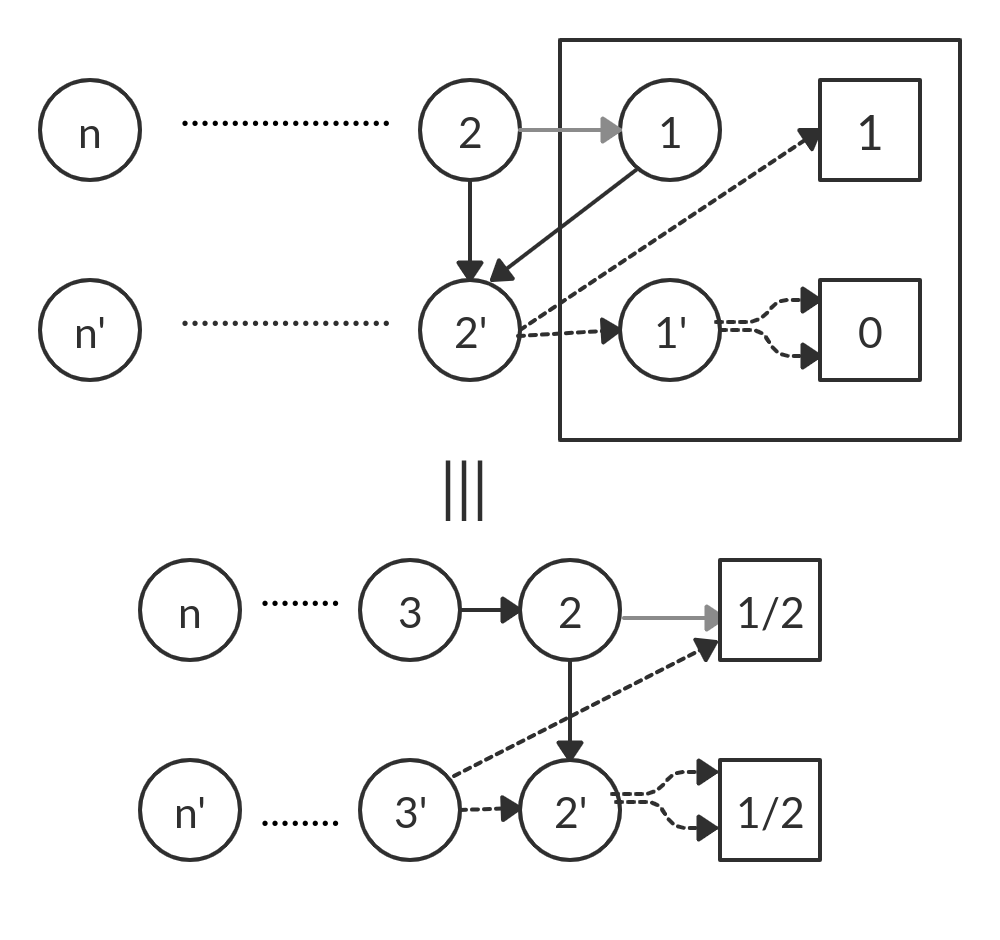}
    \caption{Reducing $\mathcal{F}_{C}(n,k)$ with respect to the policy $S_n S_{n-1}\ldots S_2 S_1 = 0,0\ldots1,k-1$}
    \label{fig:Fc_red_2}
\end{figure}

\begin{lem}
\label{lem:FC_3}
If $\pi^t = S_n S_{n-1}\ldots S_2 S_1 = 00\ldots10$ for $\mathcal{F}_{C}(n,k)$, it takes $\mathcal{N}_C(n-1,k)$ iterations to converge to optimal policy. 
\end{lem}
\begin{proof}
With respect to $\pi^t = S_n S_{n-1}\ldots S_2 S_1 = 00\ldots10$ the sinks and state 1 can be effectively reduced to new sinks with $(\alpha,\beta) = (1,\frac{1}{2})$. This reduction is shown in Figure~\ref{fig:Fc_red_3}. Invoking Lemma~\ref{lem:sink_invariant} this MDP is equivalent to $\mathcal{F}_{C}(n-1,k)$ with initial policy $S_n S_{n-1}\ldots S_2 = 00\ldots01$ and hence takes $\mathcal{N}_{C}(n-1,k)$ iterations to converge to optimal policy $S_{n-1}\ldots S_2 = 00\ldots00$. The complete policy is now $S_{n-1}\ldots S_2 S_1 = 00\ldots00$ which is also the optimal policy for $F_{C}(n,k)$
\end{proof}
\begin{figure}[htbp]
    \centering
    \includegraphics[width=\linewidth]{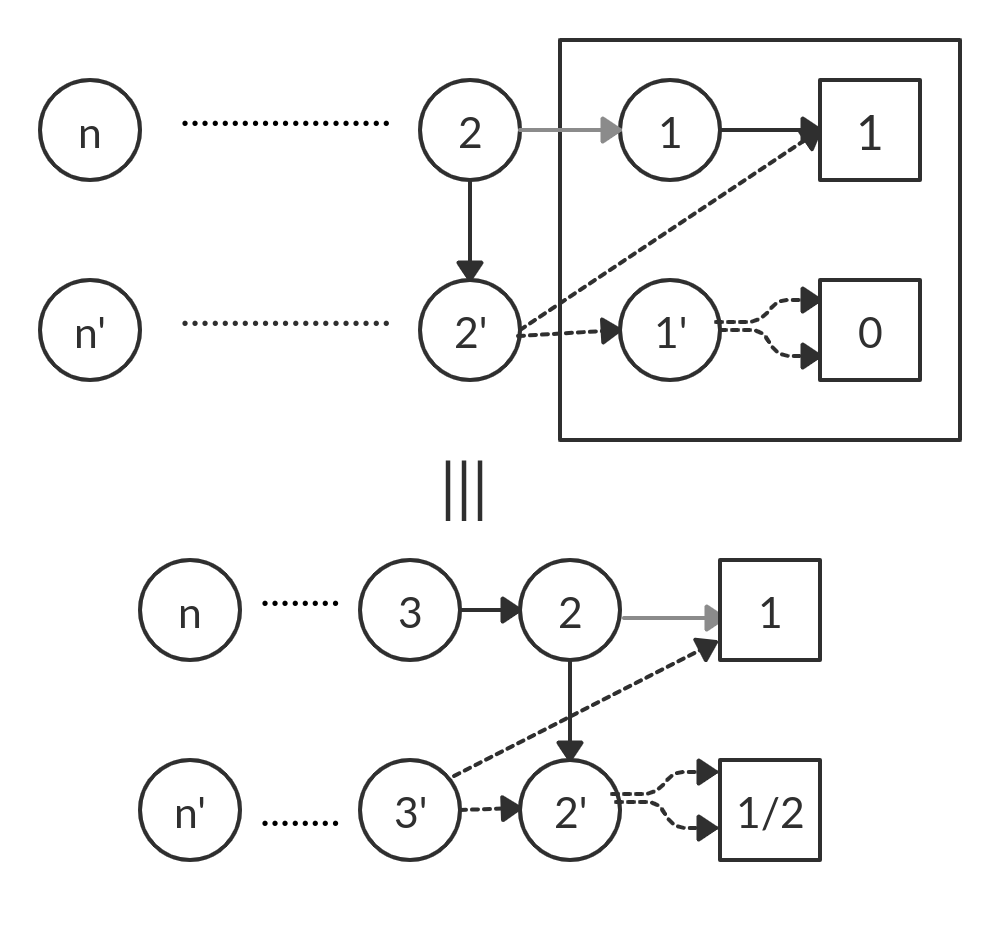}
     \caption{Reducing $\mathcal{F}_{C}(n,k)$ to  $\mathcal{F}_{C}(n-1,k)$with respect to policy $S_n S_{n-1}\ldots S_2 S_1 = 0,0\ldots1,0$}
    \label{fig:Fc_red_3}
\end{figure}

\begin{thm}
\begin{equation} \label{eq:NC}
    \mathcal{N}_C(n+1,k) = \mathcal{N}(n,k) +2+\mathcal{N}_C(n,k)
\end{equation}
\end{thm}
\begin{proof}
This can be proved by sequentially applying Lemmas 7.6-7.8 
\end{proof}

\begin{lem}
\label{lem:F_1}
For t $\mathcal{F}(n,k)$ having initial policy $S_n S_{n-1}\ldots S_2 S_1 = 00\ldots00$, it takes exactly $\mathcal{N}(n-1,k)$ iterations before policy at the first state vertex changes.
\end{lem}
\begin{proof}
Due to the switching rule of SPI, $S_1$ will only change when all the other state vertices are not improvable. Until the higher states improve, the current sinks and state 1 can be effectively reduced to new sinks with $(\alpha,\beta) = (-1,-\frac{1}{2})$. This reduction is shown in Figure \ref{fig:F_red}. Using Lemma 7.1 the resultant MDP is equivalent to $\mathcal{F}(n-1,k)$ having initial policy $S_n S_{n-1}\ldots S_2 = 00\ldots00$. This MDP takes $N(n-1,k)$ iterations to converge to the optimal policy $S_n S_{n-1}\ldots S_2 = 00\ldots01$.
By this logic after the $\mathcal{N}(n-1,k)$ iterations, the policy would be  $S_n S_{n-1}\ldots S_2 S_1 = 00\ldots10$
\end{proof}

\begin{figure}[htbp]
    \centering
    \includegraphics[width=\linewidth]{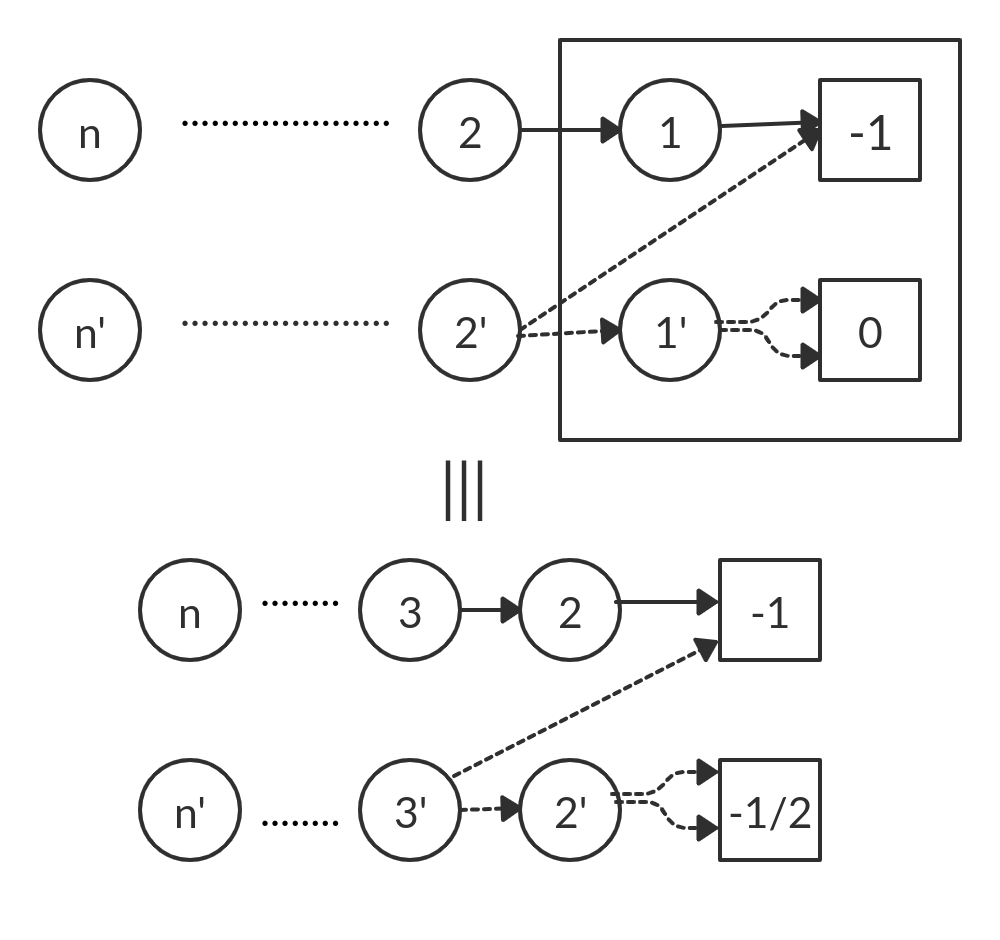}
    \caption{ Reducing $\mathcal{F}(n,k)$ to $\mathcal{F}(n-1,k)$ with respect to its initial policy}
    \label{fig:F_red}
\end{figure}

\begin{lem}
\label{lem:F_2}
If $\pi^t = S_n S_{n-1}\ldots S_2 S_1 = 00\ldots10$, the next 2 switches of simple policy iteration occur at vertex 1
\end{lem}

\begin{proof}
Using Lemma~\ref{lem:Q_1} the policy $\pi^t$ is optimal with respect to the states $n,n-1 \ldots 3,2$ and is $S_n S_{n-1}\ldots S_2 = 00\ldots01$. Vertex 1 is the only improvable state and according to Lemma~\ref{lem:Q_1} it switches to its highest indexed improvable action $k-1$. Hence the policy becomes  $S_n S_{n-1}\ldots S_2 S_1 = 0,0\ldots1,k-1$. With respect to the current policy the sinks and state 1 can be effectively reduced to new sinks with $(\alpha,\beta) = (-\frac{1}{2},-\frac{1}{2})$. This reduction is shown in Figure~\ref{fig:F_red_2}. The  action values are equal for equal sink values and hence the policy for states $n,n-1 \ldots 3,2$ is still optimal. State 1 is still the only improvable state and according to Lemma~\ref{lem:Q_1} it switches to the next improvable action k-2. This completes the proof. The policy would now be  $S_n S_{n-1}\ldots S_2 S_1 = 00\ldots1,k-2$
\end{proof}

\begin{figure}[htbp]
    \centering
    \includegraphics[width=\linewidth]{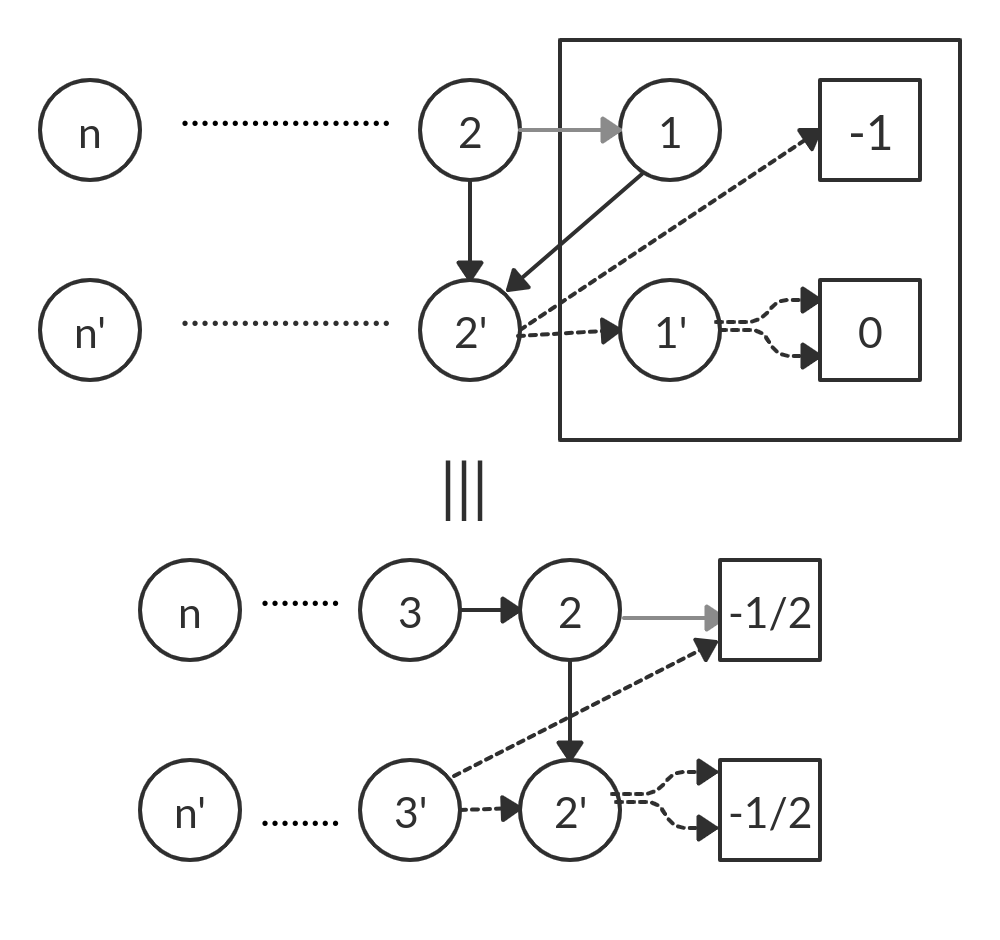}
    \caption{Reducing $\mathcal{F}(n,k)$ with respect to policy $S_n S_{n-1}\ldots S_2 S_1 = 0,0\ldots1,k-1$}
    \label{fig:F_red_2}
\end{figure}

\begin{lem}
\label{lem:F_3}
If $\pi^t = S_n S_{n-1}\ldots S_2 S_1 = 00\ldots1,k-2$ for $\mathcal{F}(n,k)$, it takes $\mathcal{N}_C(n-1,k)$ iterations before policy at the first state vertex changes. 
\end{lem}
\begin{proof}
With respect to $\pi^t = S_n S_{n-1}\ldots S_2 S_1 = 00\ldots1,k-2$ the sinks and state 1 can be effectively reduced to new sinks with $(\alpha,\beta) = (-\frac{p_{k-2}}{2},-\frac{1}{2})$. This reduction is shown in Figure~\ref{fig:F_red_3}. Invoking Lemma~\ref{lem:sink_invariant} this MDP is equivalent to $\mathcal{F}_{C}(n-1,k)$ with initial policy $S_n S_{n-1}\ldots S_2 = 00\ldots01$ and hence takes $\mathcal{N}_{C}(n-1,k)$ iterations to converge to optimal policy $S_{n-1}\ldots S_2 = 00\ldots00$. The complete policy is now $S_{n-1}\ldots S_2 S_1 = 00\ldots0,k-2$ 
\end{proof}
\begin{figure}[htbp]
    \centering
    \includegraphics[width=\linewidth]{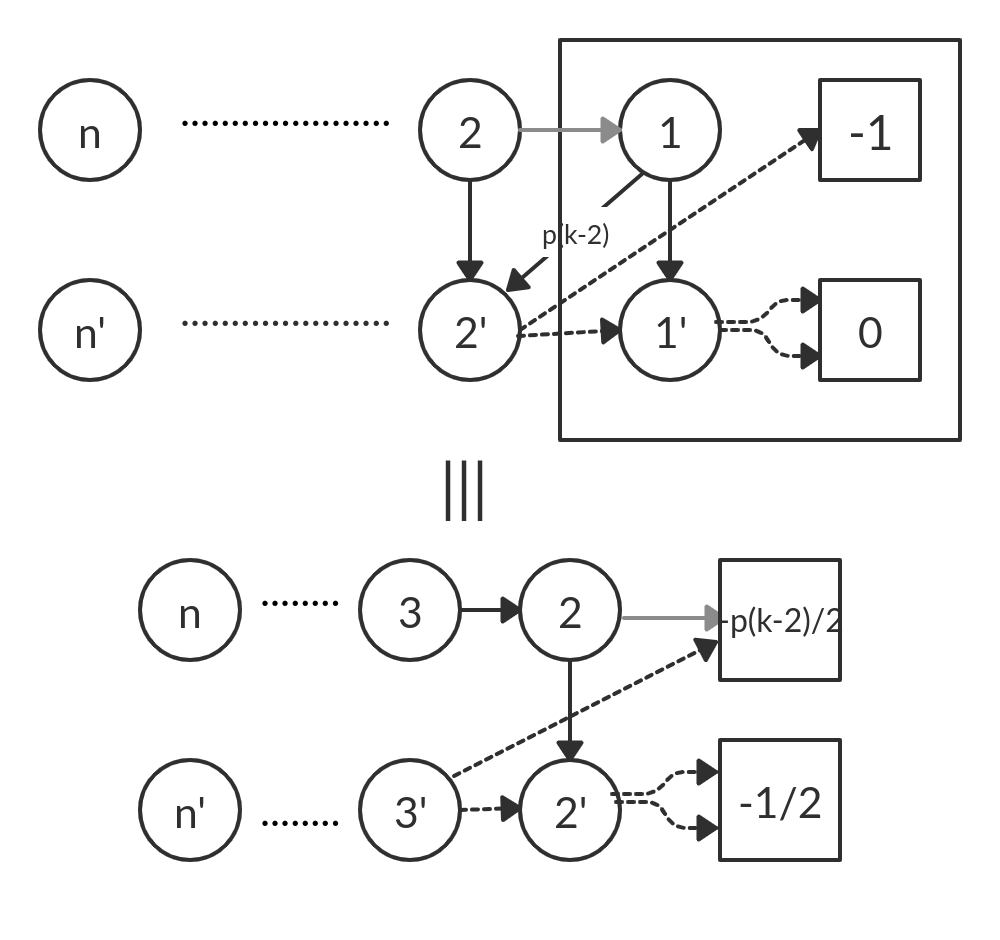}
    \caption{ Reducing $\mathcal{F}(n,k)$ to $\mathcal{F}_c(n-1,k)$ with respect to policy $S_n S_{n-1}\ldots S_2 S_1 = 0,0\ldots1,k-2$ }
    \label{fig:F_red_3}
\end{figure}

\begin{lem}
\label{lem:F_4}
If $\pi^t = S_n S_{n-1}\ldots S_2 S_1 = 0, 0, \ldots, k-2$ for $\mathcal{F}(n,k)$, it takes $k-3$ iterations to converge to optimal policy. 
\end{lem}

\begin{proof}
With respect to $\pi^t = S_n S_{n-1}\ldots S_2 S_1 = 0, 0, \ldots, k-2$  for $\mathcal{F}(n,k)$ the sinks and state 1 can be effectively reduced to new sinks with $(\alpha,\beta) = (-\frac{p_{k-2}}{2},-\frac{1}{2})$. Invoking Lemma~\ref{lem:sink_invariant} this MDP is equivalent to $\mathcal{F}_{C}(n-1,k)$ with initial policy $S_n S_{n-1}\ldots S_2 = 00\ldots00$ and hence this is already optimal. This means that only change that will happen is at state 1. Using Lemma \ref{lem:Q_1} we get an incremental change in the improvable policy. Also, at any subsequent stage $i$ the value of sink is $(\alpha,\beta) = (-\frac{p_{i}}{2},-\frac{1}{2})$, which again implies only state 1 can be improved. Hence it would take $k-3$ iterations to converge to the optimal policy $S_n S_{n-1}\ldots S_2 S_1 = 00\ldots01$

\end{proof}

\begin{thm}
\begin{equation}\label{eq:N}
  \mathcal{N}(n+1,k) = \mathcal{N}(n,k) +2+\mathcal{N}_C(n,k)+(k-3)  
\end{equation}
\end{thm}
\begin{proof}
This can be be proven by sequentially applying Lemmas~\ref{lem:F_1}-\ref{lem:F_4}.
\end{proof}

\begin{thm}[Recursive Relation]\label{thm:recursive}
$$\mathcal{N}(n+1,k) = 2\times\mathcal{N}(n,k)+2$$
\end{thm}
\begin{proof}
Subtracting eq.(\ref{eq:NC}) from eq.(\ref{eq:N}), we get the relation:
\begin{equation*}\label{eq:diffN}
    \mathcal{N}(n+1,k) - \mathcal{N}_C(n+1,k) = (k-3)
\end{equation*}
As the RHS of the above is independent of $n$, we can replace $n+1$ with $n$ to get
\begin{equation*}\label{eq:diffN}
    \mathcal{N}_C(n,k) = \mathcal{N}(n,k) - (k-3)
\end{equation*}
Substituting $\mathcal{N}_C(n,k)$ from the above equation into eq.(\ref{eq:N}) completes the proof.
\end{proof}

\begin{thm}
\begin{equation*}
 \mathcal{N}(n,k) = (3+k)2^{n-2} -2   
\end{equation*}
\end{thm}
\begin{proof}
For a fixed k, use the baseline from lemma~\ref{lem:baseline} and apply the recursive relation described by theorem~\ref{thm:recursive} to complete the proof.
\begin{align*}
    \mathcal{N}(2,k) &= k+1\\
    \mathcal{N}(3,k) &= 2(k+1)+2\\
    \mathcal{N}(4,k) &= 2^2(k+1)+2^2+2\\
    &\;\;\vdots \notag \\
    \mathcal{N}(n,k) &= 2^{n-2}(k+1)+\sum_{i=1}^{n-2}2^i\\
    &= (3+k)2^{n-2} -2\\
    &= (3+k)2^{N/2-3} -2
\end{align*}
\end{proof}

\section{\textbf{Conclusion}}
In this work, we established a generalized lower bound on the number of iterations for a $N$ state, $k$ action MDP. We demonstrated the MDP formulation and proved a lower bound of $\bigO\big((3+k)2^{N/2-3}\big)$. However, we do not reject the existence of an MDP with a tighter lower bound, say, $\bigO\big(k^{N}\big)$. Out of all of family of MDP that we constructed and verified, most of them were $\bigO\big(2^{N} + k\big)$ and a few were $\bigO\big(k.{2}^{N}\big)$. Considering the switching rule employed by our construction, and the family of other MDPs we tested, finding an MDP with a tighter lower bound would be an interesting extension to our work.

\section{\textbf{Additional Result}}
We observed that the pattern used to define multiple actions was not being followed in very last states, which had scope of improvement. A simple modification to the actions from the final state improve the baseline from $$\mathcal{N}(2,k) = k+1  \quad \forall k\geq 3$$ to $$\mathcal{N}(2,k) = 2k  \quad \forall k\geq 3$$
which futher increases the lower bound from $\bigO\big((3+k)2^{N/2-3}\big)$ to $\bigO\big((1+k)2^{N/2-2}\big)$ which was confirmed experimentally. Since the change is only at the final state, we believe that the rest of the proof will remain the same. The variation in the MDP has been shown in the Appendix Fig.\ref{fig:mdp_new}. $k1,k2$ are actions with probability $p1,p2$ respectively and $k1>k2 \implies p1>p2$.

\nocite{*}

\bibliographystyle{unsrt}

\bibliography{references}

\newpage
\section{Appendix}

We present the simulation results for all actions and states upto ten in the table below. The experimental results are in coherence with the theoretical values derived and proved.

\begin{table}[htbp]
\caption{Number of Iterations for family of MDP}
\begin{tabular}{ |c|c|c|c|c|c|c|c|c|c| } 
 \hline
  & \textbf{n=2} & \textbf{n=3} & \textbf{n=4} & \textbf{n=5} & \textbf{n=6} & \textbf{n=7} & \textbf{n=8} & \textbf{n=9} & \textbf{n=10} \\ 
 \hline
  \textbf{k=3} & 4 & 10 & 22 & 46 & 94 & 190 & 382 & 766 & 1534 \\
    \textbf{k=4} & 5 & 12 & 26 & 54 & 110 & 222 & 446 & 894 & 1790 \\
      \textbf{k=5} & 6 & 14 & 30 & 62 & 126 & 254 & 510 & 1022 & 2046 \\
        \textbf{k=6} & 7 & 16 & 34 & 70 & 142 & 286 & 574 & 1150 & 2302 \\
          \textbf{k=7} & 8 & 18 & 38 & 78 & 158 & 318 & 638 & 1278 & 2558 \\
            \textbf{k=8} & 9 & 20 & 42 & 86 & 174 & 350 & 702 & 1406 & 2814 \\
              \textbf{k=9} & 10 & 22 & 46 & 94 & 190 & 382 & 766 & 1534 & 3070\\
                \textbf{k=10} & 11 & 24 & 50 & 102 & 206 & 414 & 830 & 1662 & 3326 \\
 \hline
\end{tabular}
    \label{tab:nk}
\end{table}
\begin{figure}[htbp]
    \includegraphics[width =\linewidth]{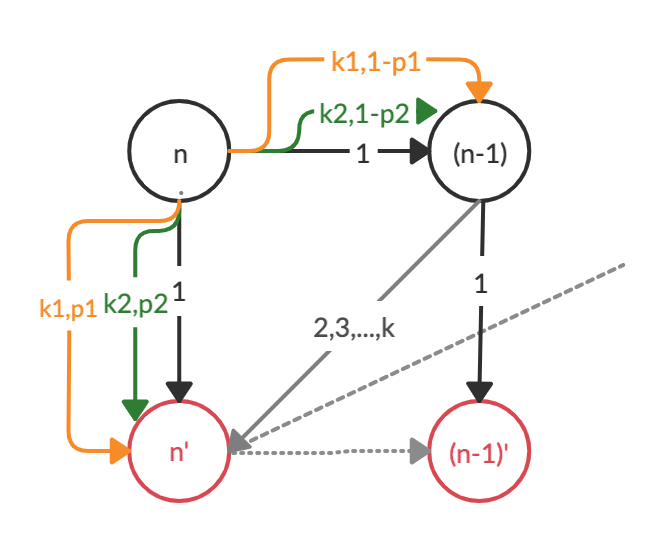}
    \caption{Improvement in MDP (Best viewed in color)}
    \label{fig:mdp_new}
\end{figure}

\end{document}